\documentclass{amsart}
\usepackage{amssymb}
\usepackage{graphicx}
\usepackage{algorithm}
\usepackage{algorithmic}

\newcommand{\Real}{\mathbb R}

\renewcommand{\b}[1]{\mathbf{#1}}
\newcommand{\bx}{\b x}
\newcommand{\by}{\b y}
\newcommand{\bigtimes}{\mathop{\times}\limits}
\newcommand{\maxnorm}[1]{\|#1\|_\infty}
\renewcommand{\hat}[1]{\widehat{#1}}

\newtheorem{theorem}{Theorem}[section]
\newtheorem{lemma}[theorem]{Lemma}
\newtheorem{definition}[theorem]{Definition}

\newcommand{\citet}[1]{\cite{#1}}
\newcommand{\citep}[1]{\cite{#1}}
\begin{document}

\title{Factored Value Iteration Converges}

\author{Istv{\'a}n Szita \and
   Andr{\'a}s L{\H{o}}rincz }
\address{E{\"o}tv{\"o}s Lor{\'a}nd University, Department of Information Systems. Budapest, Hungary, P{\'a}zm{\'a}ny P. s{\'e}t{\'a}ny 1/C, H-1117 \\
  E-mail addresses: szityu@gmail.com (Istv{\'a}n Szita), andras.lorincz@inf.elte.hu (Andr{\'a}s L{\H{o}}rincz).
  Please send correspondence to Andr{\'a}s L{\H{o}}rincz}

\maketitle

\begin{abstract}
In this paper we propose a novel algorithm, factored value
iteration (FVI), for the approximate solution of factored Markov
decision processes (fMDPs). The traditional approximate value
iteration algorithm is modified in two ways. For one, the
least-squares projection operator is modified so that it does not
increase max-norm, and thus preserves convergence. The other
modification is that we uniformly sample polynomially many samples
from the (exponentially large) state space. This way, the
complexity of our algorithm becomes polynomial in the size of the
fMDP description length. We prove that the algorithm is
convergent. We also derive an upper bound on the difference
between our approximate solution and the optimal one, and also on
the error introduced by sampling. We analyze various projection
operators with respect to their computation complexity and their
convergence when combined with approximate value iteration.

\end{abstract}

\section{Introduction}

Markov decision processes (MDPs) are extremely useful for formalizing and
solving sequential decision problems, with a wide repertoire of algorithms to
choose from \cite{Bertsekas96Neuro-Dynamic,Sutton98Reinforcement}.
Unfortunately, MDPs are subject to the `curse of dimensionality'
\cite{Bellman61Adaptive}: for a problem with $m$ state variables, the size of
the MDP grows exponentially with $m$, even though many practical problems have
polynomial-size descriptions. Factored MDPs (fMDPs) may rescue us from this
explosion, because they offer a more compact representation
\cite{Koller00Policy,Boutilier95Exploiting,Boutilier00Stochastic}. In the fMDP
framework, one assumes that dependencies can be factored to several
easy-to-handle components.

For MDPs with known parameters, there are three basic solution methods (and,
naturally, countless variants of them): value iteration, policy iteration and
linear programming (see the books of Sutton \& Barto
\cite{Sutton98Reinforcement} or Bertsekas \& Tsitsiklis
\cite{Bertsekas96Neuro-Dynamic} for an excellent overview). Out of these
methods, linear programming is generally considered less effective than the
others. So, it comes as a surprise that all effective fMDPs algorithms, to our
best knowledge, use linear programming in one way or another. Furthermore, the
classic value iteration algorithm is known to be divergent when function
approximation is used \cite{Baird95Residual,Tsitsiklis97Analysis}, which
includes the case of fMDPs, too.

In this paper we propose a variant of the approximate value iteration algorithm
for solving fMDPs. The algorithm is a direct extension of the traditional value
iteration algorithm. Furthermore, it avoids computationally expensive
manipulations like linear programming or the construction of decision trees. We
prove that the algorithm always converges to a fixed point, and that it
requires polynomial time to reach a fixed accuracy. A bound to the distance
from the optimal solution is also given.

In Section \ref{s:MDP} we review the basic concepts of Markov
decision processes, including the classical value iteration
algorithm and its combination with linear function approximation.
We also give a sufficient condition for the convergence of
approximate value iteration, and list several examples of
interest. In Section \ref{s:fMDP} we extend the results of the
previous section to fMDPs and review related works in
Section~\ref{s:literature}. Conclusions are drawn in Section
\ref{s:conc}.

\section{Approximate Value Iteration in Markov Decision Processes}
\label{s:MDP}

\subsection{Markov Decision Processes}

An MDP is characterized by a sixtuple $(\b X, A, R, P, \bx_s, \gamma)$, where
$\b X$ is a finite set of states;\footnote{Later on, we shall generalize the
concept of the state of the system. A state of the system will be a vector of
state variables in our fMDP description. For that reason, we already use the
boldface vector notation in this preliminary description.} $A$ is a finite set
of possible actions; $R: \b X \times A \to \Real$ is the reward function of the
agent, so that $R(\bx, a)$ is the reward of the agent after choosing action $a$
in state $\bx$; $P: \b X \times A \times \b X \to [0,1]$ is the transition
function so that $P(\by \mid \bx,a)$ is the probability that the agent arrives
at state $\by$, given that she started from $\bx$ upon executing action $a$;
$\bx_s \in \b X$ is the starting state of the agent; and finally, $\gamma\in
[0,1)$ is the discount rate on future rewards.

A policy of the agent is a mapping $\pi: \b X \times A \to [0,1]$ so that
$\pi(\bx,a)$ tells the probability that the agent chooses action $a$ in state
$\bx$. For any $\bx_0\in\b X$, the policy of the agent and the parameters of
the MDP determine a stochastic process experienced by the agent through the
instantiation
\[  \bx_0, a_0, r_0, \bx_1, a_1, r_1, \ldots, \bx_t, a_t, r_t, \ldots
\]
The goal is to find a policy that maximizes the expected value of the
discounted total reward. Let the value function of policy $\pi$ be
\[
  V^\pi(\bx) := E\Bigl( \sum_{t=0}^\infty \gamma^t r_t \Bigm| \bx\!=\!\bx_0
  \Bigr)
\]
and let the optimal value function be
\[
  V^*(\bx) := \max_{\pi} V^\pi(\bx)
\]
for each $\bx \in \b X$. If $V^*$ is known, it is easy to find an optimal
policy $\pi^*$, for which $V^{\pi^*} \equiv V^*$. Provided that history does
not modify transition probability distribution $P(\by|\bx,a)$ at any time
instant, value functions satisfy the famous Bellman equations
\begin{equation} \label{e:Vpi_bellman}
  V^\pi(\bx) = \sum_a \sum_{\by} \pi(\bx,a) P(\by \mid \bx,a) \Bigl( R(\bx,a) +
  \gamma V^\pi(\by) \Bigr)
\end{equation}
and
\begin{equation} \label{e:V*_bellman}
  V^*(\bx) = \max_a \sum_{\by} P(\by \mid\bx,a) \Bigl( R(\bx,a) +
  \gamma V^*(\by) \Bigr).
\end{equation}
Most algorithms that solve MDPs build upon some version of the Bellman
equations. In the following, we shall concentrate on the value iteration
algorithm.

\subsection{Exact Value Iteration}

Consider an MDP $(\b X, A, P, R, \bx_s, \gamma)$. The value iteration for MDPs
uses the Bellman equations (\ref{e:V*_bellman}) as an iterative assignment: It
starts with an arbitrary value function $V_0 : \b X \to \Real$, and in
iteration $t$ it performs the update
\begin{equation} \label{e:VI}
  V_{t+1}(\bx) := \max_a \sum_{\by \in \b X} P(\by \mid \bx,a) \Bigl( R(\bx,a) +
  \gamma V_t(\by) \Bigr)
\end{equation}
for all $\bx\in \b X$. For the sake of better readability, we shall introduce
vector notation. Let $N:=|\b X|$, and suppose that states are integers from 1
to $N$, i.e. $\b X = \{1,2,\ldots,N\}$. Clearly, value functions are equivalent
to $N$-dimensional vectors of reals, which may be indexed with states. The
vector corresponding to $V$ will be denoted as $\b v$ and the value of state
$\bx$ by ${\b v}_{\bx}$. Similarly, for each $a$ let us define the
$N$-dimensional column vector $\b r^a$ with entries $\b r^a_\bx = R(\bx,a)$ and
$N\times N$ matrix $P^a$ with entries $P^a_{\bx,\by} = P(\by \mid \bx,a)$. With
these notations, (\ref{e:VI}) can be written compactly as
\begin{equation} \label{e:VI_vector}
  \b v_{t+1} := \textbf{max}_{a\in A}  \bigl( \b r^a + \gamma P ^a
   \b v_t \bigr).
\end{equation}
Here, $\textbf{max}$ denotes the componentwise maximum operator.

It is also convenient to introduce the \emph{Bellman operator} $\mathcal T
:\Real^N \to \Real^N$ that maps value functions to value functions as
\[
  \mathcal T \b v := \textbf{max}_{a\in A}  \bigl( \b r^a + \gamma P ^a
   \b v \bigr).
\]
As it is well known, $\mathcal T$ is a max-norm contraction with
contraction factor $\gamma$: for any $\b v, \b u\in\Real^N$, $
\maxnorm{\mathcal T \b v - \mathcal T \b u} \leq \gamma
\maxnorm{\b v - \b u}$. Consequently, by Banach's fixed point
theorem, exact value iteration (which can be expressed compactly
as $\b v_{t+1} := \mathcal T \b v_t$) converges to an unique
solution $\b v^*$ from any initial vector $\b v_0$, and the
solution $\b v^*$ satisfies the Bellman equations
(\ref{e:V*_bellman}). Furthermore, for any required precision
$\epsilon>0$, $\maxnorm{\b v_t - \b v^*} \leq\epsilon$ if $t \geq
\frac{\log \epsilon}{\log \gamma} \maxnorm{\b v_0 - \b v^*}$. One
iteration costs $O(N^2\cdot|A|)$ computation steps.

\subsection{Approximate value iteration}

In this section we shall review approximate value iteration (AVI) with linear
function approximation (LFA) in ordinary MDPs. The results of this section hold
for AVI in general, but if we can perform all operations effectively on compact
representations (i.e. execution time is polynomially bounded in the number of
variables instead of the number of states), then the method can be directly
applied to the domain of factorized Markovian decision problems, underlining
the importance of our following considerations.

Suppose that we wish to express the value functions as the linear combination
of $K$ \emph{basis functions} $h_k:\b X\to \Real \ \ (k\in\{1,\ldots,K\})$,
where $K << N$. Let $H$ be the $N\times K$ matrix with entries $H_{\bx,k} =
h_k(\bx)$. Let $\b w_t \in \Real^K$ denote the weight vector of the basis
functions at step $t$. We can substitute $\b v_t = H\b w_t$ into the right hand
side (r.h.s.) of (\ref{e:VI_vector}), but we cannot do the same on the left
hand side (l.h.s.) of the assignment: in general, the r.h.s. is not contained
in the image space of $H$, so there is no such $\b w_{t+1}$ that
\[
  H\b w_{t+1} = \textbf{max}_{a\in A} \bigl( \b r^a + \gamma P^a H \b w_t \bigr).
\]
We can put the iteration into work by projecting the right-hand side into $\b
w$-space: let $\mathcal G : \Real^N \to \Real^K$ be a (possibly non-linear)
mapping, and consider the iteration
\begin{equation} \label{e:AVI_MDP}
  \b w_{t+1} := \mathcal G \bigl[\textbf{max}_{a\in A} \bigl( \b r^a + \gamma P^a H \b w_t \bigr)\bigr]
\end{equation}
with an arbitrary starting vector $\b w_0$.

\begin{lemma} \label{lem:avi}
If $\mathcal G$ is such that $H\mathcal G$ is a non-expansion, i.e., for any
$\b v, \b v' \in \Real^N$,
\[
  \maxnorm{H \mathcal G \b v - H \mathcal G \b v'} \leq \maxnorm{\b v - \b v'},
\]
then there exists a $\b w^* \in \Real^K$ such that
\[
  \b w^* = \mathcal G \bigl[ \emph{\textbf{max}}_{a\in A} \bigl( \b r^a + \gamma P^a H \b w^* \bigr)\bigr]
\]
and iteration \eqref{e:AVI_MDP} converges to $\b w^*$ from any starting point.
\end{lemma}

\begin{proof}
We can write \eqref{e:AVI_MDP} compactly as $ \b w_{t+1} =
\mathcal G \mathcal T H \b w_t$. Let $\hat{\b v}_t = H\b w_t$.
This satisfies
\begin{equation} \label{e:l1e1}
  \hat{\b v}_{t+1} = H \mathcal G \mathcal T \hat{\b v}_t.
\end{equation}
It is easy to see that the operator $H \mathcal G \mathcal T$ is a contraction:
for any $\b v, \b v' \in \Real^N$,
\begin{eqnarray*}
  \maxnorm{H\mathcal G \mathcal T \b v - H \mathcal G \mathcal T \b v'} &\leq&
         \maxnorm{\mathcal T \b v - \mathcal T \b v'}
         \leq \gamma \maxnorm{ \b v - \b v'}
\end{eqnarray*}
by the assumption of the lemma and the contractivity of $\mathcal T$.
Therefore, by Banach's fixed point theorem, there exists a vector $\hat{\b
v}^*\in \Real^N$ such that $\hat{\b v}^* = H \mathcal G \mathcal T \hat{\b
v}^*$ and iteration (\ref{e:l1e1}) converges to $\hat{\b v}^*$ from any
starting point. It is easy to see that $\b w^* = \mathcal G T \hat{\b v}^*$
satisfies the statement of the lemma.

\end{proof}

Note that if $\mathcal G$ is a linear mapping with matrix $G\in \Real^{K\times
N}$, then the assumption of the lemma is equivalent to $\maxnorm{HG}\leq 1$.

\subsection{Examples of Projections, Convergent and Divergent}
\label{ss:proj_examples}

In this section, we examine certain possibilities for choosing projection
$\mathcal G$. Let $\b v\in\Real^N$ be an arbitrary vector, and let $\b w =
\mathcal G \b v$ be its $\mathcal G$-projection. For linear operators,
$\mathcal G$ can be represented in matrix form and we shall denote it by $G$.

\textbf{Least-squares ($L_2$-)projection.}  Least-squares fitting is used
almost exclusively for projecting value functions, and the term AVI is usually
used in the sense ``AVI with least-squares projection''. In this case, $\b w$
is chosen so that it minimizes the least-squares error:
\[
  \b w := \arg\min_{\b w} \| H\b w - \b v \|^2_2.
\]
This corresponds to the linear projection $G_2=H^+$ (i.e., $\b w=H^+ \b v$),
where $H^+$ is the Moore-Penrose pseudoinverse of $H$. It is well known,
however, that this method can diverge. For an example on such divergence, see,
e.g. the book of Bertsekas \& Tsitsiklis \cite{Bertsekas96Neuro-Dynamic}. The
reason is simple: matrix $HH^+$ is a non-expansion in $L_2$-norm, but Lemma 1
requires that it should be an $L_\infty$-norm projection, which does not hold
in the general case. (See also Appendix \ref{app:proof_GL1} for illustration.)

\textbf{Constrained least-squares projection.}  One can enforce the
non-expansion property by expressing it as a constraint: Let $\b w$ be the
solution of the constrained minimization problem
\[
  \b w := \arg\min_{\b w} \| H\b w - \b v \|^2_2, \textrm{ subject to }
  \maxnorm{H\b w} \leq \maxnorm{\b v},
\]
which defines a non-linear mapping $\mathcal G_{2}^c$. This projection is
computationally highly demanding: in each step of the iteration, one has to
solve a quadratic programming problem.

\textbf{Max-norm ($L_\infty$-)projection.}  Similarly to $L_2$-projection, we
can also select $\b w$ so that it minimizes the max-norm of the residual:
\[
  \b w := \arg\min_{\b w} \| H\b w - \b v \|_\infty.
\]
The computation of $\b w$ can be transcribed into a linear programming task and
that defines the non-linear mapping $\mathcal G_{\infty}$. However, in general,
$\maxnorm{H \mathcal G_{\infty} \b v} \nleq \maxnorm{\b v}$, and consequently
AVI using iteration
\[
  \b w_{t+1} := \arg\min_{\b w} \| H\b w - \mathcal T H \b w_t \|_\infty
\]
can be divergent. Similarly to $L_2$ projection, one can also introduce a
constrained version $\mathcal G_{\infty}^c$ defined by
\[
  \mathcal G_{\infty}^c \b v := \arg\min_{\b w} \| H\b w - \b v \|_\infty, \textrm{ subject to }
  \maxnorm{H\b w} \leq \maxnorm{\b v},
\]
which can also be turned into a linear program.

It is insightful to contrast this with the approximate linear programming
method of Guestrin et al. \citet{Guestrin02Efficient}: they directly minimize
the max-norm of the Bellman error, i.e., they solve the problem
\[
  \b w^* := \arg\min_{\b w} \| H\b w - \mathcal T H \b w \|_\infty,
\]
which can be solved without constraints.

\textbf{$L_1$-norm projection.}  Let $\mathcal G_{L_1}$ be defined by
\[
  \mathcal G_{1} \b v := \arg\min_{\b w} \| H\b w - \b v \|_1.
\]
The $L_1$-norm projection also requires the solution of a linear program, but
interestingly, the projection operator $\mathcal G_1$ is a non-expansion (the
proof can be found in Appendix \ref{app:proof_GL1}).

AVI-compatible operators considered so far ($\mathcal G_{2}^c$, $\mathcal
G_{\infty}^c$ and $\mathcal G_{1}$) were non-linear, and required the solution
of a linear program or a quadratic program in each step of value iteration,
which is clearly cumbersome. On the other hand, while $G_{2} \b v = H^+ \b v$
is linear, it is also known to be incompatible with AVI
\cite{Baird95Residual,Tsitsiklis97Analysis}. Now, we shall focus on operators
that are both AVI-compatible and linear.

\textbf{Normalized linear mapping.} Let $G$ be an arbitrary $K\times N$ matrix,
and define its normalization $\mathcal N(G)$ as a matrix with the same
dimensions and entries
\[
  [\mathcal N(G)]_{i,j} := \frac{G_{i,j}}{\bigl(\sum_{j'} |H_{i,j'}|\bigr)\bigl( \sum_{i'} |G_{i',j}|\bigr)},
\]
that is, $N(G)$ is obtained from $G$ by dividing each element with the
corresponding row sum of $H$ and the corresponding column sum of $G$. All
(absolute) row sums of $H \cdot \mathcal N(G)$ are equal to 1. Therefore, (i)
$\maxnorm{H \cdot \mathcal N(G)}=1$, and (ii) $H \cdot N(G)$ is maximal in the
sense that if the absolute value of any element of $\mathcal N(G)$ increased,
then for the resulting matrix $G'$, $\maxnorm{H \cdot G'}>1$.

\textbf{Probabilistic linear mapping.} If all elements of $H$ are nonnegative
and all the row-sums of $H$ are equal, then $\mathcal N(H^T)$ assumes a
probabilistic interpretation. This interpretation is detailed in Appendix
\ref{app:prob_interpretation}.

\textbf{Normalized least-squares projection.} Among all linear operators, $H^+$
is the one that guarantees the best least-squares error, therefore we may
expect that its normalization, $\mathcal N(H^+)$ plays a similar role among
\emph{AVI-compatible} linear projections. Unless noted otherwise, we will use
the projection $\mathcal N(H^+)$ subsequently.

\subsection{Convergence properties}

\begin{lemma}
Let $\b v^*$ be the optimal value function and $\b w^*$ be the fixed point of
the approximate value iteration (\ref{e:AVI_MDP}). Then
\[
  \maxnorm{H\b w^* - \b v^*} \leq \frac{1}{1-\gamma} \maxnorm{H\mathcal G \b
  v^* - \b v^*}.
\]
\end{lemma}
\begin{proof}
For the optimal value function, $\b v^* = \mathcal T \b v^*$ holds. On the
other hand, $\b w^* = \mathcal G \mathcal T H \b w^*$. Thus,
\begin{eqnarray*}
  \maxnorm{H\b w^* - \b v^*} &=& \maxnorm{H \mathcal G \mathcal T H \b w^* - \mathcal T \b  v^*} \\
    &\leq& \maxnorm{H \mathcal G \mathcal T H \b w^* - H \mathcal G \mathcal T \b v^*} + \maxnorm{H \mathcal G \mathcal T \b v^* - \mathcal T \b  v^*} \\
    &\leq& \maxnorm{\mathcal T H \b w^* - \mathcal T \b v^*} + \maxnorm{H \mathcal G \b v^* - \b  v^*} \\
    &\leq& \gamma \maxnorm{H \b w^* - \b v^*} + \maxnorm{H \mathcal G \b v^* - \b v^*},
\end{eqnarray*}
from which the statement of the lemma follows. For the
transformations we have applied the triangle inequality, the
non-expansion property of $H\mathcal G$ and the contraction
property of $\mathcal T$.
\end{proof}

According to the lemma, the error bound is proportional to the projection error
of $\b v^*$. Therefore, if $\b v^*$ can be represented in the space of basis
functions with small error, then our AVI algorithm gets close to the optimum.
Furthermore, the lemma can be used to check \emph{a posteriori} how good our
basis functions are. One may improve the set of basis functions iteratively.
Similar arguments have been brought up by Guestrin et al.
\citet{Guestrin02Efficient}, in association with their LP-based solution
algorithm.

\section{Factored value iteration} \label{s:fMDP}

MDPs are attractive because solution time is polynomial in the number of
states. Consider, however, a sequential decision problem with $m$ variables. In
general, we need an exponentially large state space to model it as an MDP. So,
the number of states is \emph{exponential} in the size of the description of
the task. Factored Markov decision processes may avoid this trap because of
their more compact task representation.

\subsection{Factored Markov decision processes}

We assume that $\b X$ is the Cartesian product of $m$ smaller state spaces
(corresponding to individual variables):
\[
  \b X = X_1 \times X_2 \times \ldots \times X_m.
\]
For the sake of notational convenience we will assume that each $X_i$ has the
same size, $|X_1| = |X_2| = \ldots = |X_m| = n$. With this notation, the size
of the full state space is $N = |\b X| = n^m$. We note that all derivations and
proofs carry through to different size variable spaces.

A naive, tabular representation of the transition probabilities would require
exponentially large space (that is, exponential in the number of variables
$m$). However, the next-step value of a state variable often depends only on a
few other variables, so the full transition probability can be obtained as the
product of several simpler factors. For a formal description, we introduce
several notations:

For any subset of variable indices $Z \subseteq \{ 1,2,\ldots,m\}$, let $\b
X[Z] := \bigtimes_{i\in Z} X_i $, furthermore, for any $\bx \in \b X$, let
$\bx[Z]$ denote the value of the variables with indices in $Z$. We shall also
use the notation $\bx[Z]$ without specifying a full vector of values $\bx$, in
such cases $\bx[Z]$ denotes an element in $\b X[Z]$. For single-element sets
$Z=\{i\}$ we shall also use the shorthand $\bx[\{i\}] = \bx[i]$.

A function $f$ is a \emph{local-scope} function if it is defined over a
subspace $\b X[Z]$ of the state space, where $Z$ is a (presumably small) index
set. The local-scope function $f$ can be extended trivially to the whole state
space by $f(\bx) := f(\bx[Z])$. If $|Z|$ is small, local-scope functions can be
represented efficiently, as they can take only $n^{|Z|}$ different values.

Suppose that for each variable $i$ there exist neighborhood sets $\Gamma_i$
such that the value of $\bx_{t+1}[i]$ depends only on $\bx_{t}[\Gamma_i]$ and
the action $a_t$ taken. Then we can write the transition probabilities in a
factored form
\begin{equation} \label{e:Pfactored}
  P(\by \mid \bx, a) = \prod_{i=1}^n P_i(\by[i] \mid \bx[\Gamma_i], a)
\end{equation}
for each $\bx,\by\in\b X$, $a\in A$, where each factor is a local-scope
function
\begin{equation} \label{e:Pfactored2}
  P_i : \b X[\Gamma_i] \times A \times X_i \to [0,1]
  \qquad\textrm{(for all $i \in \{1,\ldots,m\}$).}
\end{equation}
We will also suppose that the reward function is the sum of $J$ local-scope
functions:
\begin{equation} \label{e:Rfactored}
  R(\bx, a) = \sum_{j=1}^{J} R_j(\bx[Z_j], a),
\end{equation}
with arbitrary (but preferably small) index sets $Z_j$, and local-scope
functions
\begin{equation} \label{e:rfactored2}
  R_j : \b X[Z_j] \times A \to \Real
  \qquad\textrm{(for all $j \in \{1,\ldots,J\}$).}
\end{equation}

To sum up, a factored Markov decision process is characterized by the
parameters $
  \Bigl( \{X_i: 1\leq i \leq m\}; A; \{R_j: 1\leq j \leq J\}; \{\Gamma_i: 1\leq i \leq n\}; \{P_i: 1\leq i \leq n\}; \bx_s;
  \gamma\Bigr),
$
where $\bx_s$ denotes the initial state.

Functions $P_i$ and $R_i$ are usually represented either as tables or dynamic
Bayesian networks. If the maximum size of the appearing local scopes is bounded
by some constant, then the description length of an fMDP is polynomial in the
number of variables $n$.

\subsubsection{Value functions}

The optimal value function is an $N=n^m$-dimensional vector. To represent it
efficiently, we should rewrite it as the sum of local-scope functions with
small domains. Unfortunately, in the general case, no such factored form exists
\citep{Guestrin02Efficient}.

However, we can still approximate $V^*$ with such an expression: let $K$ be the
desired number of basis functions and for each $k\in \{1,\ldots,K\}$, let $C_k$
be the domain set of the local-scope basis function $h_k: \b X[C_k] \to \Real$.
We are looking for a value function of the form
\begin{equation} \label{e:vhat}
  \tilde V(\bx) = \sum_{k=1}^{K} w_k \cdot h_k(\bx[C_k]).
\end{equation}

The quality of the approximation depends on two factors: the choice of the
basis functions and the approximation algorithm. Basis functions are usually
selected by the experiment designer, and there are no general guidelines how to
automate this process. For given basis functions, we can apply a number of
algorithms to determine the weights $w_k$. We give a short overview of these
methods in Section \ref{s:literature}. Here, we concentrate on value iteration.

\subsection{Exploiting factored structure in value iteration}

For fMDPs, we can substitute the factored form of the transition probabilities
(\ref{e:Pfactored}), rewards (\ref{e:Rfactored}) and the factored approximation
of the value function (\ref{e:vhat}) into the AVI formula (\ref{e:AVI_MDP}),
which yields
\begin{eqnarray*}
  \sum_{k=1}^{K} h_k(\bx[C_k]) \cdot w_{k,t+1} &\approx& \max_a \sum_{\by \in \b X}
  \Bigl( \prod_{i=1}^m P_i(\by[i] \mid \bx[\Gamma_i],a) \Bigr) \cdot \\
  && \cdot \Bigl( \sum_{j=1}^{J} R_j(\bx[Z_j], a) +
  \gamma \sum_{k'=1}^{K} h_{k'}(\by[C_{k'}]) \cdot w_{k',t}  \Bigr).
\end{eqnarray*}
By rearranging operations and exploiting that all occurring functions have a
local scope, we get
\begin{eqnarray}
  &&\sum_{k=1}^{K} h_k(\bx[C_k]) \cdot w_{k,t+1} = \mathcal G_k \max_a
  \Biggl[ \sum_{j=1}^{J} R_j(\bx[Z_j], a)  \nonumber \\
  && + \gamma \sum_{k'=1}^{K} \sum_{\by[C_{k'}] \in \b X[C_{k'}]} \Bigl(\prod_{i\in C_{k'}} P_i(\by[i] \mid
  \bx[\Gamma_i],a)\Bigr) h_{k'}(\by[C_{k'}]) \cdot w_{k',t} \Biggr] \label{e:full_overdetermined_system}
\end{eqnarray}
for all $\bx \in \b X$. We can write this update rule more compactly in vector
notation. Let
\[
  \b w_t := (w_{1,t}, w_{2,t}, \ldots, w_{K,t}) \in \Real^{K},
\]
and let $H$ be an $|\b X| \times K$ matrix containing the values of the basis
functions. We index the rows of matrix $H$ by the elements of $\b X$:
\[
  H_{\bx,k} := h_k(\bx[C_k]).
\]
Further, for each $a\in A$, let $B^a$ be the $|\b X| \times K$ \emph{value
backprojection} matrix defined as
\[
  B^a_{\bx,k} := \sum_{\by[C_{k}] \in \b X[C_{k}]} \Bigl(\prod_{i\in C_{k}} P_i(\by[i] \mid
  \bx[\Gamma_i],a)\Bigr) h_{k}(\by[C_{k}])
\]
and for each $a$, define the reward vector $\b r^a \in \Real^{|\b X|}$ by
\[
  \b r^a_\bx := \sum_{j=1}^{n_r} R_j(\bx[Z_j], a).
\]
Using these notations, (\ref{e:full_overdetermined_system}) can be rewritten as
\begin{equation}
  \b w_{t+1} := \mathcal G \textbf{max}_{a\in A} \Bigl( \b r^a + \gamma B^a \b w_t
  \Bigr).
\end{equation}

Now, all entries of $B^a$, $H$ and $\b r^a$ are composed of local-scope
functions, so any of their individual elements can be computed efficiently.
This means that the time required for the computation is exponential in the
sizes of function scopes, but only polynomial in the number of variables,
making the approach attractive. Unfortunately, the matrices are still
exponentially large, as there are exponentially many equations in
(\ref{e:full_overdetermined_system}). One can overcome this problem by sampling
as we show below.

\subsection{Sampling}

To circumvent the problem of having exponentially many equations, we select a
random subset $\hat{\b X} \subseteq \b X$ of the original state space so that
$|\hat{\b X}| = \textrm{poly}(m)$, consequently, solution time will scale
polynomially with $m$. On the other hand, we will select a sufficiently large
subset so that the remaining system of equations is still over-determined. The
necessary size of the selected subset is to be determined later: it should be
as small as possible, but the solution of the reduced equation system should
remain close to the original solution with high probability. For the sake of
simplicity, we assume that the projection operator $\mathcal G$ is linear with
matrix $G$. Let the sub-matrices of $G$, $H$, $B^a$ and $\b r^a$ corresponding
to $\hat{\b X}$ be denoted by $\hat G$, $\hat H$, $\hat B^a$ and $\hat{\b
r}^a$, respectively. Then the following value update
\begin{equation} \label{e:what_iteration}
  \b w_{t+1} :=\hat G \cdot \textbf{max}_{a\in A} \Bigl( \hat{\b r}^a + \gamma \hat B^a \b w_t
  \Bigr)
\end{equation}
can be performed effectively, because these matrices have polynomial size. Now
we show that the solution from sampled data is close to the true solution with
high probability.

\begin{theorem} \label{thm:w*_sample}
Let $\b w^*$ be the unique solution of $ \b w^* = G \mathcal T H \b w^*$ of an FMDP, and
let ${\b w}'$ be the solution of the corresponding equation with sampled
matrices, ${\b w}' = \hat G \mathcal T \hat H {\b w}'$. Suppose that the projection
matrix $G$ has a factored structure, too. Then iteration
\eqref{e:what_iteration} converges to $\b w'$, furthermore, for a suitable
constant $\,\Xi$ (depending polynomially on $n^z$, where $z$ is the maximum cluster size),
and for any $\epsilon,\delta>0$, $\maxnorm{\b w^* - \b w'}
\leq \epsilon$ holds with probability at least $1-\delta$, if the sample size
satisfies $\displaystyle N_1 \geq \Xi \frac{m^2}{\epsilon^2} \log
\frac{m}{\delta} $.
\end{theorem}
The proof of Theorem~\ref{thm:w*_sample} can be found in Appendix
\ref{app:proof_sampling}. The derivation is closely related to the work of
Drineas and colleagues \citet{Drineas06Fast,Drineas06Sampling}, although we
use the infinity-norm instead of the $L_2$-norm. A more important difference
is that we can exploit the factored structure, gaining an exponentially better bound. The
resulting \emph{factored value iteration} algorithm is summarized in
Table~\ref{alg:featuregeneration}.

\begin{algorithm}[h!]
\caption{Factored value iteration with a linear projection matrix
$G$.} \label{alg:featuregeneration}
\begin{algorithmic}
     \STATE \textit{\% inputs:}
     \STATE \textit{\% factored MDP, $\mathcal M = \bigl( \{X_i\}_{i=1}^m; A; \{R_j\}_{j=1}^{J}; \{\Gamma_i\}_{i=1}^m; \{P_i\}_{i=1}^m; \bx_s;
     \gamma\bigr)$}
     \STATE  \% basis functions, $\{h_k\}_{k=1}^{K}$
     \STATE  \% required accuracy, $\epsilon$

     \STATE  $N_0 := $ number of samples
     \STATE $\hat{\b X} :=$ uniform random $N_0$-element subset of $\b X$
     \STATE create $\hat H$ and $\hat G$
     \STATE create $\hat B^a = \widehat{P^a H}$ and $\hat{\b r}^a$ for each $a\in A$
     \STATE $\b w_0 = \b 0$,   $t :=0$
     \REPEAT{}
         \STATE $\displaystyle \b w_{t+1} :=\hat G \cdot \max_{a\in A} \Bigl( \hat{\b r}^a + \gamma \hat B^a \b w_t \Bigr)$
         \STATE $\Delta_t := \| \b w_{t+1} - \b w_t \|_\infty$
         \STATE $t := t+1$
     \UNTIL{$\Delta_t \geq \epsilon$}
     \RETURN $\b w_t$
\end{algorithmic}
\end{algorithm}

%
%

%
%
%

\section{Related work} \label{s:literature}

The exact solution of factored MDPs is infeasible. The idea of representing a
large MDP using a factored model was first proposed by Koller \& Parr
\cite{Koller00Policy} but similar ideas appear already in the works of
Boutilier, Dearden, \& Goldszmidt
\cite{Boutilier95Exploiting,Boutilier00Stochastic}. More recently, the
framework (and some of the algorithms) was extended to fMDPs with hybrid
continuous-discrete variables \cite{Kveton06Solving} and factored partially
observable MDPs \cite{Sallans02Reinforcement}. Furthermore, the framework has
also been applied to structured MDPs with alternative representations, e.g.,
relational MDPs \cite{Guestrin03Generalizing} and first-order MDPs
\cite{Sanner05Approximate}.

\subsection{Algorithms for solving factored MDPs}

There are two major branches of algorithms for solving fMDPs: the first one
approximates the value functions as decision trees, the other one makes use of
linear programming.


Decision trees (or equivalently, decision lists) provide a way to represent the
agent's policy compactly. Koller \& Parr \citet{Koller00Policy} and Boutilier
et al. \citet{Boutilier95Exploiting,Boutilier00Stochastic} present algorithms
to evaluate and improve such policies, according to the policy iteration
scheme. Unfortunately, the size of the policies may grow exponentially even
with a decision tree representation
\citep{Boutilier00Stochastic,Liberatore02Size}.


The exact Bellman equations (\ref{e:V*_bellman}) can be transformed to an
equivalent linear program with $N$ variables $\{V(\bx) : \bx \in \b X \}$ and
$N\cdot |A|$ constraints:
\begin{eqnarray*}
  \textrm{maximize: } && \sum_{\bx \in \b X} \alpha(\bx) V(\bx) \\
  \textrm{subject to } && V(\bx) \leq R(\bx,a) +  \gamma \sum_{\bx' \in\b X}
    P(\bx' \mid \bx,a) V(\bx'), \quad\textrm{$(\forall \bx\in\b X, a\in A)$.}
\end{eqnarray*}
Here, weights $\alpha(\bx)$ are free parameters and can be chosen freely in the
following sense: the optimum solution is $V^*$ independent of their choice,
provided that each of them is greater than 0. In the approximate linear
programming approach, we approximate the value function as a linear combination
of basis functions (\ref{e:vhat}), resulting in an approximate LP with $K$
variables $\{w_k : 1\leq k \leq K\}$ and $N\cdot |A|$ constraints:
\begin{eqnarray}
  \textrm{maximize: } && \sum_{k=1}^{K} \sum_{\bx \in \b X} w_k \cdot \alpha(\bx) h_k(\bx[C_k])  \label{e:alp}\\
  \textrm{subject to } && \sum_{k=1}^{K} w_k \cdot h_k(\bx[C_k]) \leq \nonumber \\ &&  \hspace{15mm} \leq R(\bx,a) +
  \gamma \sum_{k'=1}^{K} w_{k'} \sum_{\bx' \in\b X}
    P(\bx' \mid \bx,a) \cdot h_{k'}(\bx'[C_{k'}]). \nonumber
\end{eqnarray}
Both the objective function and the constraints can be written in compact
forms, exploiting the local-scope property of the appearing functions.

Markov decision processes were first formulated as LP tasks by Schweitzer and
Seidmann \citet{Schweitzer85Generalized}. The approximate LP form is due to de
Farias and van Roy \citet{Farias01Approximate}. Guestrin et al.
\citet{Guestrin02Efficient} show that the maximum of local-scope functions can
be computed by rephrasing the task as a non-serial dynamic programming task and
eliminating variables one by one. Therefore, (\ref{e:alp}) can be transformed
to an equivalent, more compact linear program. The gain may be exponential, but
this is not necessarily so in all cases: according to Guestrin et al.
\cite{Guestrin02Efficient}, ``as shown by Dechter \citet{Dechter99Bucket}, [the
cost of the transformation] is exponential in the induced width of the cost
network, the undirected graph defined over the variables $X_1; \ldots; X_n$,
with an edge between $X_l$ and $X_m$ if they appear together in one of the
original functions $f_j$. The complexity of this algorithm is, of course,
dependent on the variable elimination order and the problem structure.
Computing the optimal elimination order is an NP-hard problem
\citep{Arnborg87Complexity} and elimination orders yielding low induced tree
width do not exist for some problems.'' Furthermore, for the approximate LP
task (\ref{e:alp}), the solution is no longer independent of $\alpha$ and the
optimal choice of the $\alpha$ values is not known.

The approximate LP-based solution algorithm is also due to Guestrin et al.
\citet{Guestrin02Efficient}. Dolgov and Durfee \citet{Dolgov06Symmetric} apply
a primal-dual approximation technique to the linear program, and report
improved results on several problems.


The approximate policy iteration algorithm
\citep{Koller00Policy,Guestrin02Efficient} also uses an approximate LP
reformulation, but it is based on the policy-evaluation Bellman equation
(\ref{e:Vpi_bellman}). Policy-evaluation equations are, however, linear and do
not contain the maximum operator, so there is no need for the second, costly
transformation step. On the other hand, the algorithm needs an explicit
decision tree representation of the policy. Liberatore \cite{Liberatore02Size}
has shown that the size of the decision tree representation can grow
exponentially.

\subsubsection{Applications}

Applications of fMDP algorithms are mostly restricted to artificial test
problems like the problem set of Boutilier et al. \cite{Boutilier00Stochastic},
various versions of the \textsc{SysAdmin} task
\citep{Guestrin02Efficient,Dolgov06Symmetric,Patrascu02Greedy} or the New York
driving task \citep{Sallans02Reinforcement}.

Guestrin, Koller, Gearhart and Kanodia \citet{Guestrin03Generalizing} show that
their LP-based solution algorithm is also capable of solving more practical
tasks: they consider the real-time strategy game \emph{FreeCraft}. Several
scenarios are modelled as fMDPs, and solved successfully. Furthermore, they
find that the solution generalizes to larger tasks with similar structure.

\subsubsection{Unknown environment}

The algorithms discussed so far (including our FVI algorithm) assume that all
parameters of the fMDP are known, and the basis functions are given. In the
case when only the factorization structure of the fMDP is known but the actual
transition probabilities and rewards are not, one can apply the factored
versions of E$^3$ \citep{Kearns99Efficient} or R-max
\citep{Guestrin02Algorithm-Directed}.

Few attempts exist that try to obtain basis functions or the structure of the
fMDP automatically. Patrascu et al. \citet{Patrascu02Greedy} select basis
functions greedily so that the approximated Bellman error of the solution is
minimized. Poupart et al. \citet{Poupart02Piecewise} apply greedy selection,
too, but their selection criteria are different: a decision tree is constructed
to partition the state space into several regions, and basis functions are
added for each region. The approximate value function is piecewise linear in
each region. The metric they use for splitting is related to the quality of the
LP solution.

\subsection{Sampling}

Sampling techniques are widely used when the state space is immensely large.
Lagoudakis and Parr \citet{Lagoudakis03Least-Squares} use sampling without a
theoretical analysis of performance, but the validity of the approach is
verified empirically. De Farias and van Roy \citet{Farias04Constraint} give a
thorough overview on constraint sampling techniques used for the linear
programming formulation. These techniques are, however, specific to linear
programming and cannot be applied in our case.

The work most similar to ours is that of Drineas et al.
\citet{Drineas06Sampling,Drineas06Fast}. They investigate the least-squares
solution of an overdetermined linear system, and they prove that it is
sufficient to keep polynomially many samples to reach low error with high
probability. They introduce a non-uniform sampling distribution, so that the
variance of the approximation error is minimized. However, the calculation of
the probabilities requires a complete sweep through all equations.


\section{Conclusions} \label{s:conc}

In this paper we have proposed a new algorithm, factored value
iteration, for the approximate solution of factored Markov
decision processes. The classical approximate value iteration
algorithm is modified in two ways. Firstly, the least-squares
projection operator is substituted with an operator that does not
increase max-norm, and thus preserves convergence. Secondly,
polynomially many samples are sampled uniformly from the
(exponentially large) state space. This way, the complexity of our
algorithm becomes polynomial in the size of the fMDP description
length. We prove that the algorithm is convergent and give a bound
on the difference between our solution and the optimal one. We
also analyzed various projection operators with respect to their
computation complexity and their convergence when combined with
approximate value iteration. To our knowledge, this is the first
algorithm that (1) provably converges in polynomial time and (2)
avoids linear programming.

\section*{Acknowledgements}

The authors are grateful to Zolt{\'a}n Szab{\'o} for calling our attention to the
articles of Drineas et al. \citet{Drineas06Fast,Drineas06Sampling}. This
research has been supported by the EC FET `New and Emergent World models
Through Individual, Evolutionary, and Social Learning' Grant (Reference Number
3752). Opinions and errors in this manuscript are the author's responsibility,
they do not necessarily reflect those of the EC or other project members.

\appendix

\section{Proofs}

\subsection{Projections in various norms}
\label{app:proof_GL1}

We wish to know whether $  \b w_0 = \arg\min_{\b w}  \| H\b w - \b
v \|_p $  implies $\maxnorm{H\b w_0} \leq \maxnorm{\b v}$ for
various values of $p$. Specifically, we are interested in the
cases when $p \in \{1,2,\infty\}$. Fig.~\ref{fig:projection}
indicates that the implication does not hold for $p=2$ or
$p=\infty$, only for the case $p=1$. Below we give a rigorous
proof for these claims.

\begin{figure}
\centering
\includegraphics[width=16cm]{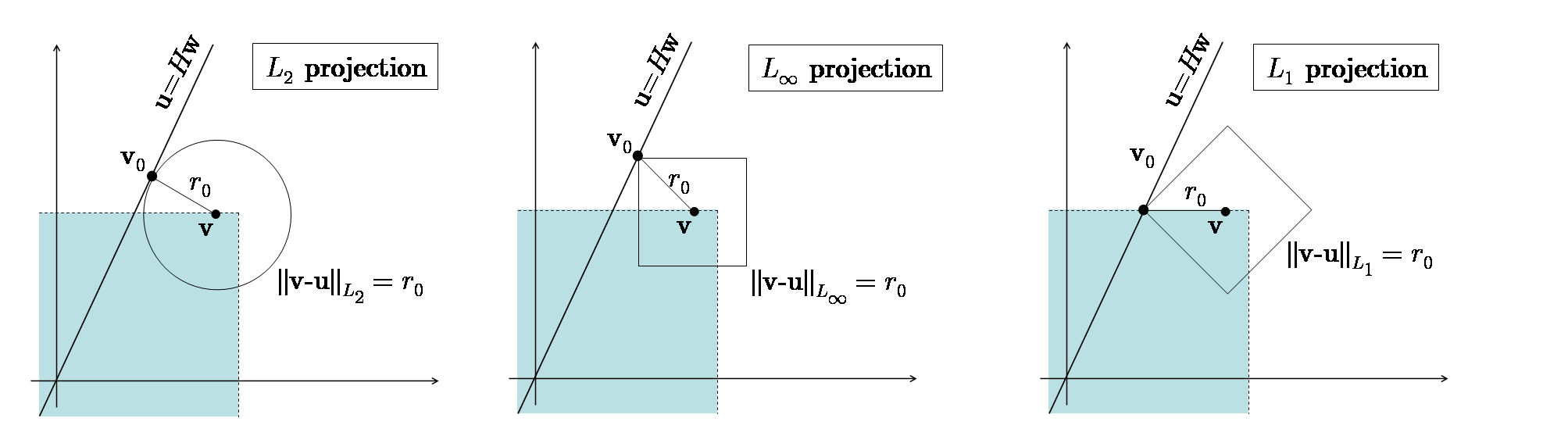}
 \caption{Projections in various norms. The vector $\b v$ is projected onto the
image space of $H$, i.e., the subspace defined by $\b u = H\b w$. Consider the
smallest sphere around $\b v$ (in the corresponding norm) that touches the
subspace $\b u = H\b w$ (shown in each figure). The radius $r_0$ of this sphere
is the distance of $\b v$ from the subspace, and the tangent point $\b v_0$
(which is not necessarily unique for $L_1$ projection) is the projection of $\b
v$. For this point, $\b v_0 = H \b w_0$ holds. The shaded region indicates the
region $\{ \b u : \maxnorm{\b u} \leq \maxnorm{\b v} \}$. To ensure the
convergence of FVI, the projected vector $\b v_0$ must fall into the shaded
region.  }\label{fig:projection}
\end{figure}

Consider the example $\b v = [1, 1]^T \in \Real^2$, $H=[1, 2]^T$,
$\b w\in \Real^1$. For these values easy calculation shows that $
  \maxnorm{H[\b w_0]_{L_2}} = 6/5 $
and $  \maxnorm{H[\b w_0]_{L_\infty}} = 4/3,$ i.e., $\maxnorm{H\b
w_0} \nleq \maxnorm{\b v}$ for both cases. For $p=1$, we shall
prove the following lemma:
\begin{lemma}
If $\b w_0 = \arg\min_{\b w}  \| H\b w - \b v \|_1$, then $\maxnorm{H\b w_0}
\leq \maxnorm{\b v}$.
\end{lemma}
\begin{proof}
Let $\b z := H \b w_0 \in \Real^N$. If there are multiple solutions to the
minimization task, then consider the (unique) $\b z$ vector with minimum
$L_2$-norm. Let $r:= \| \b z-\b v\|_1$ and let $S(\b v, r)$ be the $L_1$-sphere
with center $\b v$ and radius $r$ (this is an $N$-dimensional \emph{cross
polytope} or \emph{orthoplex}, a generalization of the octahedron).

Suppose indirectly that $\maxnorm{\b z} > \maxnorm{\b v}$. Without
loss of generality we may assume that $z_1$ is the coordinate of
$\b z$ with the largest absolute value, and that it is positive.
Therefore, $z_1 > \maxnorm{\b v}$. Let $\b e_i$ denote the
$i^{\scriptsize \textrm{th}}$ coordinate vector ($1\leq i \leq
N$), and let $\b z_\delta = \b z - \delta \b e_1$. For small
enough $\delta$, $S(\b z_\delta,\delta)$ is a cross polytope such
that (a) $S(\b z_\delta,\delta) \subset S(\b v, r)$, (b) $\forall
\b z' \in S(\b z_\delta,\delta)$, $\maxnorm{\b z'} > \maxnorm{\b
v}$, (c) $\forall \epsilon>0$ sufficiently small, $(1-\epsilon) \b
z \in S(\b z_\delta,\delta)$. The first two statements are
trivial. For the third statement, note that $\b z$ is a vertex of
the cross polytope $S(\b z_\delta,\delta)$. Consider the cone
whose vertex is $\b z$ and its edges are the same as the edges of
$S(\b z_\delta,\delta)$ joining $\b z$. It is easy to see that the
vector pointing from $\b z$ to the origo is contained in this
cone: for each $1<i\leq N$, $|z_i| \leq z_1$ (as $z_1$ is the
largest coordinate). Consequently, for small enough $\epsilon$,
$\b z -\epsilon \b z \in S(\b z_\delta,\delta)$.

Fix such an $\epsilon$ and let $\b q = (1-\epsilon)\b z$. This
vector is (a) contained in the image space of $H$ because
$H[(1-\epsilon) \b w] = \b q$; (b) $ \| \b q-\b v\|_1 \leq \| \b
z-\b v\|_1 = r$. The vector $\b z$ was chosen so that it has the
smallest $L_1$-norm in the image space of $H$, so the inequality
cannot be sharp, i.e., $\| \b q-\b v\|_1 = r$. However, $\|\b
q\|_2 = (1-\epsilon) \|\b z \|_2 < \|\b z \|_2$ with strict
inequality, which contradicts our assumption about $\b z$, thus
completing the proof. \end{proof}

\subsection{Probabilistic interpretation of $N(H^T)$} \label{app:prob_interpretation}
\begin{definition}
The basis functions $\{h_k\}_{k=1}^{n_b}$ have the \emph{uniform covering (UC)}
property, if all row sums in the corresponding $H$ matrix are identical:
\[
  \sum_{k=1}^{n_b} H_{\bx, k} = B \qquad\textrm{for all $\bx\in \b X$},
\]
and all entries are nonnegative. Without loss of generality we may assume that
all rows sum up to 1, i.e., $H$ is a stochastic matrix.
\end{definition}

We shall introduce an auxiliary MDP $\overline{\mathcal M}$ such that exact
value iteration in $\overline{\mathcal M}$ is identical to the approximate
value iteration in the original MDP $\mathcal M$. Let $\b S$ be an $K$-element
state space with states $\b s_1, \ldots, \b s_{K}$. A state $\b s$ is
considered a discrete observation of the true state of the system, $\b x\in \b
X$.

The action space $A$ and the discount factor $\gamma$ are identical to the
corresponding items of $\mathcal M$, and an arbitrary element $\b s_0\in \b S$
is selected as initial state. In order to obtain the transition probabilities,
let us consider how one can get from observing $\b s$ to observing $\b s'$ in
the next time step: from observation $\b s$, we can infer the hidden state $\b
x$ of the system; in state $\b x$, the agent makes action $a$ and transfers to
state $\b x'$ according to the original MDP; after that, we can infer the
probability that our observation will be $\b s'$, given the hidden state $\b
x'$. Consequently, the transition probability $\overline{P}(\b s'\mid \b s,a)$
can be defined as the total probability of all such paths:
\[
  \overline{P}(\b s'\mid \b s,a) := \sum_{\bx, \bx' \in \b X} \Pr(\bx \mid \b s)
   \Pr(\bx' \mid \bx, a) \Pr(\b s' \mid \bx).
\]
Here the middle term is just the transition probability in the original MDP,
the rightmost term is $H_{\bx,\b s}$, and the leftmost term can be rewritten
using Bayes' law (assuming a uniform prior on $\bx$):
\begin{eqnarray*}
  \Pr(\bx \mid \b s) &=& \frac{\Pr(\b s \mid \bx) \Pr(\bx)}{\sum_{\bx''\in\b X} \Pr(\b s \mid \bx'')
      \Pr(\bx'')}
  = \frac{H_{\bx,\b s}\cdot \frac{1}{|\b X|}}{\sum_{\bx''\in\b X} H_{\bx'',\b s}
      \cdot \frac{1}{|\b X|}} = \frac{H_{\bx,\b s}}{\sum_{\bx''\in\b X} H_{\bx'',\b s}}.
\end{eqnarray*}
Consequently,
\[
  \overline{P}(\b s'\mid \b s,a) = \sum_{\bx, \bx' \in \b X} \frac{H_{\bx,\b s}}{\sum_{\bx''\in\b X} H_{\bx'',\b s}}
   P(\bx' \mid \bx, a) H_{\bx,\b s} = \bigl[ \mathcal N(H)^T P^a H \bigr]_{\b s, \b s'}.
\]
The rewards can be defined similarly:
\[
  \overline{R}(\b s, a) := \sum_{\bx\in \b X} \Pr(\bx \mid \b s) R(\bx, a) = \bigl[ \mathcal N(H)^T \b r^a \bigr]_{\b s}.
\]
It is easy to see that approximate value iteration in $\mathcal M$ corresponds
to exact value iteration in the auxiliary MDP $\overline{\mathcal M}$.

\subsection{The proof of the sampling theorem (theorem \ref{thm:w*_sample})} \label{app:proof_sampling}

First we prove a useful lemma about approximating the product of two large
matrices. Let $A \in \Real^{m\times n}$ and $B\in\Real^{n\times k}$ and let
$C=A\cdot B$. Suppose that we sample columns of $A$ uniformly at random (with
repetition), and we also select the corresponding rows of $B$. Denote the
resulting matrices with $\hat A$ and $\hat B$. We will show that $A\cdot B
\approx c \cdot \hat A \cdot \hat B$, where $c$ is a constant scaling factor
compensating for the dimension decrease of the sampled matrices. The following
lemma is similar to Lemma 11 of \citet{Drineas06Fast}, but here we estimate the
infinity-norm instead of the $L_2$-norm.

\begin{lemma} \label{lem:sampling}
Let $A \in \Real^{m\times N}$, $B\in\Real^{N\times k}$ and $C=A\cdot B$. Let
$N'$ be an integer so that $1\leq N' \leq N$, and for each
$i\in\{1,\ldots,N'\}$, let $r_i$ be an uniformly random integer from the
interval $[1,N]$. Let $\hat A \in \Real^{m \times N'}$ be the matrix whose
$i^{\scriptsize \textrm{th}}$ column is the $r_i{}^{\scriptsize \textrm{th}}$
column of $A$, and denote by $\hat B$ the $N'\times k$ matrix that is obtained
by sampling the rows of $B$ similarly. Furthermore, let
\[
  \hat C = \frac{N}{N'} \hat A \cdot \hat B = \frac{N}{N'} \sum_{i=1}^{N'} A_{*,r_i} B_{r_i,*}.
\]
Then, for any $\epsilon,\delta>0$, $\maxnorm{\hat
C-C} \leq \epsilon N \maxnorm{A} \maxnorm{B^T}$ with probability at least
$1-\delta$, if the sample size satisfies $N' \geq
\frac{2m^2}{\epsilon^2} \log \frac{2km}{\delta}$.
\end{lemma}
\begin{proof}
We begin by bounding individual elements of the matrix $\hat C$: consider the
element
\[
  \hat C_{pq} = \frac{N}{N'} \sum_{i=1}^{N'} A_{p,r_i} B_{r_i,q}.
\]
Let $\mathcal C_{pq}$ be the discrete probability distribution determined by
mass points $\{ A_{p,i}\cdot B_{i,q} \mid 1\leq i \leq N \}$.
Note that $\hat C_{pq}$ is essentially the sum of $N'$ random variables drawn
uniformly from distribution $\mathcal C_{pq}$. Clearly,
\begin{eqnarray*}
  |A_{pi} B_{iq}| &\leq& \max_{ij} |A_{ij}| \max_{ij} |B_{ij}|
    \leq \max_{i}\sum_j |A_{ij}| \max_{i}\sum_j |B_{ij}| = \maxnorm{A} \maxnorm{B},
\end{eqnarray*}
so we can apply Hoeffding's
inequality to obtain
\[
  \Pr\Bigl( \left|\frac{\sum_{i=1}^{N'} A_{p,r_i} B_{r_i,q} }{N'} -
  \frac{\sum_{j=1}^{N} A_{p,j} B_{j,q} }{N} \right| >\epsilon_1 \Bigr)
  < 2 \exp \Bigl( - \frac{N' \epsilon_1^2}{2\maxnorm{A}^2 \maxnorm{B}^2} \Bigr),
\]
or equivalently,
\[
  \Pr\Bigl( \left|\frac{N }{N'}[\hat A \hat B]_{pq} -
   [AB]_{pq} \right| > N\epsilon_1 \Bigr)
  < 2 \exp \Bigl( - \frac{N' \epsilon_1^2}{2\maxnorm{A}^2 \maxnorm{B}^2} \Bigr),
\]
where $\epsilon_1 > 0$ is a constant to be determined later. From this, we can
bound the row sums of $\hat C - C$:
\[
  \Pr\Bigl( \sum_{p=1}^m \left|\hat C_{pq} -
  C_{pq} \right| > m\cdot N\epsilon_1 \Bigr)
  < 2 m\exp \Bigl( - \frac{N' \epsilon_1^2}{2\maxnorm{A}^2 \maxnorm{B}^2} \Bigr),
\]
which gives a bound on $\maxnorm{\hat C - C}$. This is the maximum of these row
sums:
\[
\Pr\Bigl( \maxnorm{\hat C - C} > m N\epsilon_1 \Bigr) =
  \Pr\Bigl( \max_q \sum_{p=1}^m \left|\hat C_{pq} -
  C_{pq} \right| > mN\epsilon_1 \Bigr)
  < 2 km \exp \Bigl( - \frac{N' \epsilon_1^2}{2\maxnorm{A}^2 \maxnorm{B}^2} \Bigr).
\]

Therefore, by substituting $\epsilon_1 = \epsilon \maxnorm{A} \maxnorm{B}/
m$, the statement of the lemma is satisfied if
 $ 2 km \exp \Bigl( - \frac{N' \epsilon^2}{2m^2} \Bigr) \leq
 \delta, $
i.e, if $  N' \geq \frac{2m^2}{\epsilon^2} \log
\frac{2km}{\delta}.$
\end{proof}

If both $A$ and $B$ are structured, we can sharpen the lemma to give a much better
(potentially exponentially better) bound. For this, we need the following definition:

For any index set $Z$, a matrix $A$ is called $Z$-local-scope matrix, if each column of $A$
represents a local-scope function with scope $Z$.

\begin{lemma}
Let $A^T$ and $B$ be local-scope matrices with scopes $Z_1$ and $Z_2$, and let $N_0 = n^{|Z_1|+|Z_2|}$,
and apply the random row/column selection procedure of the previous lemma. Then, for any $\epsilon,\delta>0$, $\maxnorm{\hat
C-C} \leq \epsilon N_0 \maxnorm{A} \maxnorm{B}$ with probability at least
$1-\delta$, if the sample size satisfies $N' \geq
\frac{2m^2}{\epsilon^2} \log \frac{2km}{\delta}$.
\end{lemma}
\begin{proof}
Fix a variable assignment $\bx[Z_1\cup Z_2]$ on the domain $Z_1 \cup Z_2$ and consider the rows of
$A$ that correspond to a variable assignment compatible to $\bx[Z_1\cup Z_2]$, i.e., they are
identical to it for components $Z_1\cup Z_2$ and are arbitrary on
\[
W:= \{1,2,\ldots,m\}\setminus (Z_1\cup
Z_2).
\]
It is easy to see that all of these rows are identical because of the local-scope property. The same
holds for the columns of $B$. All the equivalence classes of rows/columns have cardinality
\[
N_1 : = n^{|W|} = N/N_0.
\]
Now let us define the $m\times N_0$ matrix $A'$ so that only one column is kept from each
equivalence class, and define the $N_0 \times k$ matrix $B'$ similarly, by omitting rows. Clearly,
\[
  A\cdot B = N_1 A' \cdot B',
\]
and we can apply the sampling lemma to the smaller matrices $A'$ and $B'$ to get
that for any $\epsilon,\delta>0$ and sample size $N' \geq \frac{2m^2}{\epsilon^2} \log \frac{2km}{\delta}$,
with probability at least $1-\delta$,
\[
\maxnorm{\widehat{A'\cdot B'}-A'\cdot B'} \leq \epsilon N_0 \maxnorm{A'} \maxnorm{B'}.
\]
Exploiting the fact that the max-norm of a matrix is the maximum of row norms, $\maxnorm{A'} =
\maxnorm{A}/N_1$ and $\maxnorm{B'} =
\maxnorm{B}$, we can multiply both sides to get
\[
\maxnorm{N_1 \widehat{A'\cdot B'}-A\cdot B} \leq \epsilon N_0 N_1 \maxnorm{A} /N_1 \maxnorm{B} = \epsilon N_0 \maxnorm{A} \maxnorm{B},
\]
which is the statement of the lemma.
\end{proof}

Note that if the scopes $Z_1$ and $Z_2$ are small, then the gain compared to the previous lemma
can be exponential.

\begin{lemma}
Let $A = A_1 + \ldots + A_p$ and $B = B_1 + \ldots + B_q$ where all $A_i$ and $B_j$ are local-scope
matrices with domain size at most $z$, and let $N_0 = n^z$. If we apply the random row/column selection procedure,
then for any $\epsilon,\delta>0$, $\maxnorm{\hat
C-C} \leq \epsilon N_0 pq \max_i \maxnorm{A_i} \max_j \maxnorm{B_j}$ with probability at least
$1-\delta$, if the sample size satisfies $N' \geq
\frac{2m^2}{\epsilon^2} \log \frac{2km}{\delta}$.
\end{lemma}
\begin{proof}
\[
 \maxnorm{\hat C-C} \leq \sum_{i=1}^p \sum_{j=1}^q \maxnorm{\widehat{A_i\cdot B_j}- A_i \cdot B_j}.
\]
For each individual product term we can apply the previous lemma. Note that we can use the same
row/column samples for each product, because independence is required only \emph{within} single
matrix pairs. Summing the right-hand sides gives the statement of the lemma.
\end{proof}

Now we can complete the proof of Theorem \ref{thm:w*_sample}:

\begin{proof}
\begin{eqnarray*}
  \maxnorm{\b w^* - \b w'} &=& \maxnorm{G \mathcal T H  \b w^* - \hat G \mathcal T \hat H \b w'} \\
  &\leq& \maxnorm{G \mathcal T H  \b w^* - \hat G \mathcal T \hat H \b w^*} + \maxnorm{\hat G \mathcal T \hat H  \b w^* - \hat G \mathcal T \hat H \b w'} \\
  &\leq& \maxnorm{G \mathcal T H  \b w^* - \hat G \mathcal T \hat H \b w^*} + \gamma \maxnorm{\b w^* - \b w'},
\end{eqnarray*}
i.e., $\maxnorm{\b w^* - \b w'} \leq \frac{1}{1-\gamma}\maxnorm{G
\mathcal T H \b w^* - \hat G \mathcal T \hat H \b w^*}$. Let
$\pi_0$ be the greedy policy with respect to the value function $H
\b w^*$. With its help, we can rewrite $\mathcal T H \b w^*$ as a
linear expression: $
 \mathcal T H \b w^* = \b r^{\pi_0} + \gamma P^{\pi_0} H \b w^*. $
Furthermore, $\mathcal T$ is a componentwise operator, so we can
express its effect on the downsampled value function as $
 \mathcal T \hat H \b w^* = \hat{\b r}^{\pi_0} + \gamma \widehat{P^{\pi_0} H} \b
 w^*.$
Consequently,
\begin{eqnarray*}
\maxnorm{G \mathcal T H \b w^* - \hat G \mathcal T \hat H \b w^*} \leq
  \maxnorm{G \b r^{\pi_0} - \hat G \hat{\b r}^{\pi_0}} +
  \gamma \maxnorm{G P^{\pi_0} H  - \hat G \widehat{P^{\pi_0} H} } \maxnorm{\b w^*}
\end{eqnarray*}
Applying the previous lemma two times, we get that with
probability greater than $1-\delta_1$, $
  \maxnorm{G \b r^{\pi_0} - \hat G \hat{\b r}^{\pi_0}} \leq \epsilon_1
  C_1$
if $N' \geq \frac{2m^2}{\epsilon_1^2} \log \frac{2m}{\delta_1}$
and with probability greater than $1-\delta_2$, $
  \maxnorm{G P^{\pi_0} H  - \hat G \widehat{P^{\pi_0} H} } \leq \epsilon_2
  C_2 $
if $N' \geq \frac{2m^2}{\epsilon_2^2} \log \frac{2m^2}{\delta_2}$; where $C_1$ and $C_2$ are
constants depending polynomially on $N_0$ and the norm of the component local-scope functions, but independent
of $N$.

Using the notation $M= \frac{1}{1-\gamma} \Bigl( C_1 + \gamma C_2 \maxnorm{\b w^*}
\Bigr)$, $\epsilon_1 = \epsilon_2 = \epsilon/M$, $\delta_1 =
\delta_2 = \delta/2$ and $\Xi = M^2$ proves the theorem.
\end{proof}

Informally, this theorem tells that the required number of samples grows
quadratically with the desired accuracy $1/\epsilon$ and logarithmically with
the required certainty $1/\delta$, furthermore, the dependence on the number of
variables $m$ is slightly worse than quadratic. This means that even if the
number of equations is exponentially large, i.e., $N = O(e^m)$, we can select a
polynomially large random subset of the equations so that with high
probability, the solution does not change very much.

\bibliographystyle{plain}
\bibliography{rl}

\begin{thebibliography}{10}

\bibitem{Arnborg87Complexity}
Stefan Arnborg, Derek~G. Corneil, and Andrzej Proskurowski.
\newblock Complexity of finding embeddings in a k-tree.
\newblock {\em SIAM Journal on Algebraic and Discrete Methods}, 8(2):277--284,
  1987.

\bibitem{Baird95Residual}
L.C. Baird.
\newblock Residual algorithms: Reinforcement learning with function
  approximation.
\newblock In {\em ICML}, pages 30--37, 1995.

\bibitem{Bellman61Adaptive}
Richard~E. Bellman.
\newblock {\em Adaptive Control Processes}.
\newblock Princeton University Press, Princeton, NJ., 1961.

\bibitem{Bertsekas96Neuro-Dynamic}
D.P. Bertsekas and J.N. Tsitsiklis.
\newblock {\em Neuro-Dynamic Programming}.
\newblock Athena Scientific, 1996.

\bibitem{Boutilier95Exploiting}
C.~Boutilier, R.~Dearden, and M.~Goldszmidt.
\newblock Exploiting structure in policy construction.
\newblock In {\em IJCAI}, pages 1104--1111, 1995.

\bibitem{Boutilier00Stochastic}
Craig Boutilier, Richard Dearden, and Moises Goldszmidt.
\newblock Stochastic dynamic programming with factored representations.
\newblock {\em Artificial Intelligence}, 121(1-2):49--107, 2000.

\bibitem{Farias01Approximate}
D.P. de~Farias and B.~van Roy.
\newblock Approximate dynamic programming via linear programming.
\newblock In {\em NIPS}, pages 689--695, 2001.

\bibitem{Farias04Constraint}
D.P. de~Farias and B.~van Roy.
\newblock On constraint sampling in the linear programming approach to
  approximate dynamic programming.
\newblock {\em Mathematics of Operations Research}, 29(3):462--478, 2004.

\bibitem{Dechter99Bucket}
R.~Dechter.
\newblock Bucket elimination: A unifying framework for reasoning.
\newblock {\em AIJ}, 113(1-2):41--85, 1999.

\bibitem{Dolgov06Symmetric}
Dmitri~A. Dolgov and Edmund~H. Durfee.
\newblock Symmetric primal-dual approximate linear programming for factored
  {MDP}s.
\newblock In {\em Proceedings of the 9th International Symposium on Artificial
  Intelligence and Mathematics (AI\&M 2006)}, 2006.

\bibitem{Drineas06Fast}
P.~Drineas, R.~Kannan, and M.W. Mahoney.
\newblock Fast {M}onte {C}arlo algorithms for matrices {I}: Approximating
  matrix multiplication.
\newblock {\em SIAM Journal of Computing}, 36:132--157, 2006.

\bibitem{Drineas06Sampling}
P.~Drineas, M.W. Mahoney, and S.~Muthukrishnan.
\newblock Sampling algorithms for l2 regression and applications.
\newblock In {\em SODA}, pages 1127--1136, 2006.

\bibitem{Guestrin02Efficient}
C.~Guestrin, D.~Koller, R.~Parr, and S.~Venkataraman.
\newblock Efficient solution algorithms for factored {MDP}s.
\newblock {\em JAIR}, 19:399--468, 2002.

\bibitem{Guestrin02Algorithm-Directed}
C.~Guestrin, R.~Patrascu, and D.~Schuurmans.
\newblock Algorithm-directed exploration for model-based reinforcement learning
  in factored mdps.
\newblock In {\em ICML}, pages 235--242, 2002.

\bibitem{Guestrin03Generalizing}
Carlos Guestrin, Daphne Koller, Chris Gearhart, and Neal Kanodia.
\newblock Generalizing plans to new environments in relational {MDP}s.
\newblock In {\em Eighteenth International Joint Conference on Artificial
  Intelligence}, 2003.

\bibitem{Kearns99Efficient}
M.J. Kearns and D.~Koller.
\newblock Efficient reinforcement learning in factored {MDP}s.
\newblock In {\em Proceedings of the 16th International Joint Conference on
  Artificial Intelligence}, pages 740--747, 1999.

\bibitem{Koller00Policy}
D.~Koller and R.~Parr.
\newblock Policy iteration for factored {MDP}s.
\newblock In {\em UAI}, pages 326--334, 2000.

\bibitem{Kveton06Solving}
Branislav Kveton, Milos Hauskrecht, and Carlos Guestrin.
\newblock Solving factored {MDPs} with hybrid state and action variables.
\newblock {\em Journal of Artificial Intelligence Research}, 27:153--201, 2006.

\bibitem{Lagoudakis03Least-Squares}
M.G. Lagoudakis and R.~Parr.
\newblock Least-squares policy iteration.
\newblock {\em JMLR}, 4:1107--1149, 2003.

\bibitem{Liberatore02Size}
P.~Liberatore.
\newblock The size of {MDP} factored policies.
\newblock In {\em AAAI}, pages 267--272, 2002.

\bibitem{Patrascu02Greedy}
Relu Patrascu, Pascal Poupart, Dale Schuurmans, Craig Boutilier, and Carlos
  Guestrin.
\newblock Greedy linear value-approximation for factored markov decision
  processes.
\newblock In {\em Proceedings of the 18th National Conference on Artificial
  intelligence}, pages 285--291, 2002.

\bibitem{Poupart02Piecewise}
Pascal Poupart, Craig Boutilier, Relu Patrascu, and Dale Schuurmans.
\newblock Piecewise linear value function approximation for factored mdps.
\newblock In {\em Proceedings of the 18th National Conference on AI}, 2002.

\bibitem{Sallans02Reinforcement}
Brian Sallans.
\newblock {\em Reinforcement Learning for Factored Markov Decision Processes}.
\newblock PhD thesis, University of Toronto, 2002.

\bibitem{Sanner05Approximate}
Scott Sanner and Craig Boutilier.
\newblock Approximate linear programming for first-order {MDP}s.
\newblock In {\em Proceedings of the 21th Annual Conference on Uncertainty in
  Artificial Intelligence (UAI)}, pages 509--517, 2005.

\bibitem{Schweitzer85Generalized}
P.J. Schweitzer and A.~Seidmann.
\newblock Generalized polynomial approximations in {M}arkovian decision
  processes.
\newblock {\em Journal of Mathematical Analysis and Applications},
  110(6):568--582, 1985.

\bibitem{Sutton98Reinforcement}
R.S. Sutton and A.G. Barto.
\newblock {\em Reinforcement {L}earning: An {I}ntroduction}.
\newblock MIT Press, 1998.

\bibitem{Tsitsiklis97Analysis}
John~N. Tsitsiklis and Benjamin~Van Roy.
\newblock An analysis of temporal-difference learning with function
  approximation.
\newblock {\em IEEE Transactions on Automatic Control}, 42(5):674--690, 1997.

\end{thebibliography}

\end{document}